\documentclass{article} 
\usepackage{iclr,times}

\usepackage{amssymb}
\usepackage{amsmath,amsfonts,bm}

\mathchardef\mhyphen="2D

\newcommand\normbig[1]{\big\lVert#1\big\rVert}

\newcommand\normf[1]{\left\lVert#1\right\rVert_F}
\newcommand\normfbig[1]{\big\lVert#1\big\rVert_F}

\def\1{\bm{1}}

\newcommand{\real}{\mathbb{R}}

\definecolor{titlepagecolor}{cmyk}{75,68,67,90}
\definecolor{titlepagecolor2}{rgb}{1.0, 0.08, 0.58}
\definecolor{emerald}{rgb}{0.31, 0.78, 0.47}
\definecolor{deeppink}{HTML}{D14064}
\definecolor{lowpink}{HTML}{ffe6ec}
 
\definecolor{lowblue}{HTML}{E1EBFE}

\let\oldforall\forall
\renewcommand{\forall}{\oldforall\, }

\usepackage{color}   

\definecolor{mylightbluetitle}{RGB}{60,113,183}
\definecolor{structurecolorblue}{RGB}{60,113,183}
\definecolor{structurecolorgreen}{RGB}{63,145,182}
\colorlet{structurecolor}{structurecolorblue}
\definecolor{structurecolorelegant}{RGB}{60,113,183}
\definecolor{structurecolorlt}{RGB}{31,119,185}

\definecolor{structurecolorHighTheoremBlue}{RGB}{220,227,248}
\definecolor{structurecolorHighTheoremGreen}{RGB}{188,222,231}
\colorlet{structurecolorHighTheorem}{structurecolorHighTheoremBlue}

\definecolor{winestain}{rgb}{0.5,0,0}
\definecolor{mydarkblue}{rgb}{0,0.08,0.45}
\definecolor{mydarkred}{rgb}{0.70,0.00,0.00}
\definecolor{mydarkgreen}{rgb}{0.00,0.30,0.00}
\definecolor{mydarkyellow}{RGB}{197,151,13}
\definecolor{mydarkpurple}{RGB}{149,18,192}
\definecolor{mydarkgray}{RGB}{64,64,64}

\definecolor{color0}  {RGB}{174,225,254} 
\definecolor{color1}  {RGB}{220,227,248} 
\definecolor{color2}  {RGB}{28,130,185} 
\definecolor{color3}  {RGB}{255,253,250} 
\definecolor{colormiddleright}  {RGB}{245,253,250} 
\definecolor{colorbottomleft}  {RGB}{255,243,250} 
\definecolor{coloruppermiddle}  {RGB}{255,253,230} 
\definecolor{colormiddleleft}  {RGB}{255,244,237}
\definecolor{colorcr}  {RGB}{249,253,232} 
\definecolor{colorreduction}  {RGB}{255,235,254} 
\definecolor{colorqr}  {RGB}{254,221,199} 
\definecolor{colorbiconjugate}  {RGB}{251,149,161} 
\definecolor{colorsvd}  {RGB}{215,247,235} 
\definecolor{colorupperright}  {RGB}{239,246,251} 
\definecolor{colorspectral}  {RGB}{206,226,243} 
\definecolor{colorbottomright}  {RGB}{220,224,236} 
\definecolor{coloreigenvalue}  {RGB}{197,203,224} 
\definecolor{colorcp} {RGB}{217, 234, 186} 
\definecolor{colorcpborder} {RGB}{233, 243, 216} 
\definecolor{colorupperleft}  {RGB}{235,243,240} 
\definecolor{colorsemidefinite}  {RGB}{217,232,226} 
\definecolor{colormiddle} {RGB}{235, 240,255}
\definecolor{colorlu}  {RGB}{220,227,255} 
\definecolor{colorals}  {RGB}{240,230,255} 
\definecolor{coloralsbkg}  {RGB}{248,243,255} 
\definecolor{canaryyellow}{rgb}{1.0, 0.75, 0.0}
\definecolor{bluepigment}{rgb}{0.0, 0.0, 1.0}
\definecolor{canarypurple}{RGB}{208, 13, 241}
\definecolor{colorGreenOcre}{RGB}{51,102,0} 
\definecolor{colorBlue2}{RGB}{200,207,248}
\definecolor{shadecolor}{gray}{0.75}

\newcommand{\rank}{\mathrm{rank}}

\newcommand{\trace}{\mathrm{tr}}

\newcommand{\bLambda}{\boldsymbol\Lambda}

\newcommand{\widetildebA}{\widetilde{\bm{A}}}

\newcommand{\widetildebW}{\widetilde{\bm{W}}}
\newcommand{\widetildebX}{\widetilde{\bm{X}}}

\newcommand{\widetildebZ}{\widetilde{\bm{Z}}}

\newcommand{\bA}{\bm{A}}

\newcommand{\bB}{\bm{B}}

\newcommand{\bI}{\bm{I}}

\newcommand{\bL}{\bm{L}}

\newcommand{\bO}{\bm{O}}

\newcommand{\bP}{\bm{P}}

\newcommand{\bQ}{\bm{Q}}

\newcommand{\bS}{\bm{S}}

\newcommand{\bu}{\bm{u}}

\newcommand{\bv}{\bm{v}}

\newcommand{\bx}{\bm{x}}
\newcommand{\bX}{\bm{X}}

\DeclareMathAlphabet{\mathsfit}{\encodingdefault}{\sfdefault}{m}{sl}
\SetMathAlphabet{\mathsfit}{bold}{\encodingdefault}{\sfdefault}{bx}{n}

\usepackage{flushend} 
\usepackage{balance}
 
\usepackage{array,multirow,graphicx}
\usepackage{changes}
\usepackage{booktabs}
\usepackage{graphicx} 
\usepackage{xparse}
\usepackage{hyperref}
\usepackage{url}
\usepackage{subfigure} 
\usepackage{cancel}

\usepackage{sidecap}
\sidecaptionvpos{figure}{t}
\usepackage{verbatimbox}
\usepackage[labelfont=bf]{caption} 
\usepackage{wrapfig}

\usepackage{algorithm}
\usepackage{algpseudocode}
\usepackage{graphics}
\usepackage{epsfig}
\usepackage{graphicx}
\usepackage{color}   
\definecolor{winestain}{rgb}{0.5,0,0}
\definecolor{ocre}{RGB}{51,102,0} 
\definecolor{colorBlue2}{RGB}{200,207,248}
\definecolor{mydarkblue}{rgb}{0,0.08,0.45}
\definecolor{mylightbluetext}{rgb}{0,0.08,0.45}
\usepackage{hyperref}
\hypersetup{
	colorlinks=true, 
	linktoc=all,     
	linkcolor=mydarkblue,  
	anchorcolor=blue,
	citecolor=mydarkblue,
}
\PassOptionsToPackage{numbers, compress}{natbib}
\usepackage[numbers]{natbib}

\renewenvironment{thebibliography}[1]{
	\begin{oldthebibliography}{#1}
		\setlength{\itemsep}{0em}
		\setlength{\parskip}{0.07em}
	}
	{
	\end{oldthebibliography}
}

\makeatletter
\renewcommand\section{\@startsection {section}{1}{\z@}%
	{-0.1\baselineskip}
	{0.1\baselineskip}
	{\normalfont\Large\bfseries}}
\makeatother

\makeatletter
\newcommand{\paragrapharrow}{%
	\@startsection{paragraph}{4}{\z@}%
	{3.25ex \@plus1ex \@minus.2ex}%
	{-1em}%
	{\normalfont\normalsize\bfseries$\blacktriangleright$\ }}
\makeatother

\title{Large Language Model Compression via the\\ Nested Activation-Aware Decomposition}

\author{
Jun Lu\thanks{Correspondence to: Jun Lu $<$jun.lu.locky@gmail.com$>$.}, \,\,
Tianyi Xu, \;Bill Ding, \;David Li, \; Yu Kang\thanks{BA Inc.} 
\\
}

\newtheorem{theorem}{Theorem}

\newtheorem{proof}{Proof}
\newcommand{\BlackBox}{\rule{1.5ex}{1.5ex}} 
\renewenvironment{proof}{\par\noindent{\bf Proof\ }}{\hfill\BlackBox\\[2mm]}

\begin{document}

\maketitle

\begin{abstract}
In this paper, we tackle  the critical challenge of compressing large language models (LLMs) to facilitate their practical deployment and broader  adoption. We introduce a novel post-training compression paradigm that focuses on low-rank decomposition of LLM weights. 
Our analysis identifies two main challenges in this task: the variability in LLM activation distributions and handling unseen activations from different datasets and models.

To address these challenges, we propose a nested activation-aware framework (NSVD) for LLMs, a training-free approach designed to enhance the accuracy of low-rank decompositions by managing activation outliers through transforming the weight matrix based on activation distribution and the original weight matrix. This method allows for the absorption of outliers into the transformed weight matrix, improving decomposition accuracy. 
Our comprehensive evaluation across \textcolor{black}{eight} datasets and \textcolor{black}{six} models from three distinct LLM families  demonstrates the superiority of NSVD over current state-of-the-art methods, especially at medium to large compression ratios or in multilingual and multitask settings.

\end{abstract}

\section{Introduction}
By training on vast amounts of textual data, large language models (LLMs) have demonstrated exceptional capabilities across a wide range of tasks. These models can learn rich linguistic structures and contextual information, performing excellently in various applications such as natural language understanding, text generation, question-answering systems, machine translation, sentiment analysis, and more \citep{brown2020language, thirunavukarasu2023large, achiam2023gpt, merchant2023scaling, lu2024improving}. As technology advances, the scale of LLMs continues to grow, leading to performance improvements. However, this also raises the requirements for computational resources and storage. Consequently, an important area of current research focuses on effectively compressing these models to reduce their resource consumption while maintaining high performance. This involves studying and applying various compression methods, such as neural network and transformer pruning \citep{frantar2023sparsegpt, mirzadeh2023relu}, knowledge distillation \citep{tunstall2023zephyr},   weight quantization \citep{dettmers2022gpt3, frantar2022gptq, lin2024awq}, and others, aiming to find the optimal balance between efficiency and effectiveness.

Among these techniques, low-rank matrix decomposition remains relatively unexplored but holds great promise \citep{hsu2022language, yuan2023asvd, wang2024svd}. This technique involves approximating the weight matrices in neural networks or transformer structures with lower rank matrices, thereby reducing the overall model size. Given the vast number of parameters in LLMs, low-rank decomposition can significantly decrease memory usage. Moreover, it can enhance the efficiency of already compressed models by further compressing quantized or pruned models.
Despite its potential, singular value decomposition (SVD) or other matrix decomposition for LLM compression has not been fully realized. Although some SVD-based methods like FWSVD, ASVD, and SVD-LLM have been proposed \citep{hsu2022language, yuan2023asvd, wang2024svd}, they still suffer from  performance degradation at medium to high compression ratios or datasets with significantly  different activations (i.e., ``overfitting" happens during the post-training procedure). This limitation stems from two main issues: imprecise data preprocessing, where existing strategies fail to strike a balance between the activation-aware compression loss and the original matrix compression loss;
and the unawareness of different activations for different model architectures, different datasets, or different tasks.

This paper introduces the \textit{nested activation-aware framework (NSVD)}, a  nested SVD-based compression method for LLMs that addresses the aforementioned issues. NSVD features two key components: truncation-aware data whitening, which ensures a direct correlation between singular values and compression loss, allowing for minimal loss when truncating singular values; and a two-way decomposition for adhering to the original weight matrix.

We evaluate NSVD against three other SVD-based LLM compression methods---standard SVD,  and two other ASVD approaches (ASVD-0 and ASVD-I; which will be introduced in the sequel)---across \textcolor{black}{eight} datasets and \textcolor{black}{six} models from three LLM families (LLaMA, OPT, and Mistral) at different scales. Our findings highlight that NSVD consistently outperforms these methods across all tested scenarios, especially at medium to large compression ratios (30\% to 50\%) or in multilingual and multitask settings.

\section{Related Work}\label{section:rel_work}

\paragrapharrow{Low-rank approximation.}
In the context of low-rank matrix approximation, two types of problems emerge due to the interaction between rank and error: the  \textit{fixed-precision approximation problem} and the \textit{fixed-rank approximation problem}. In the fixed-precision approximation problem, given a matrix $\bA$ and a  tolerance $\epsilon$, the goal is to find a matrix $\widetildebA$ with rank $r = r(\epsilon)$ such that $\normbig{\bA-\widetildebA} \leq \epsilon$ in an appropriate matrix norm. 
Conversely, in the fixed-rank approximation problem, one seeks  a matrix $\widetildebA$ with a fixed rank $k$ that minimizes the error $\normbig{\bA-\widetildebA}$ \citep{kishore2017literature, martinsson2019randomized, lu2022matrix}. This paper focuses on the latter. 

\paragrapharrow{Pruning and quantization.}

The compression of LLMs has become an essential area of research aimed at mitigating their significant computational and memory demands  \citep{yuan2023asvd, wang2024svd}. Driven by this need, various strategies have emerged, broadly categorized into weight quantization \citep{dettmers2022gpt3, frantar2022gptq, lin2024awq}, network and transformer pruning \citep{frantar2023sparsegpt, mirzadeh2023relu}, knowledge distillation \citep{tunstall2023zephyr}, and low-rank approximation \citep{hsu2022language, yuan2023asvd, wang2024svd}. Among these, post-training methods that avoid resource-intensive retraining processes are particularly noteworthy, including unstructured pruning, structured pruning, quantization, and low-rank approximation.

Unstructured pruning sets individual weights to zero without changing the model's shape, exemplified by SparseGPT \citep{frantar2023sparsegpt}, which prunes less important weight elements using the inversion of the Hessian matrix. However, the irregular sparsification achieved through unstructured pruning often fails to achieve the desired speedup or memory savings. Structured pruning, in contrast, removes entire channels or components from LLMs, making it easier to implement on hardware but potentially leading to significant accuracy degradation, especially under high compression ratios as seen with LLM-Pruner \citep{ma2023llm}.

Quantization techniques reduce the precision of weight matrices, offering limited compression options typically ranging from 3 to 8 bits, which may not fully utilize available memory budgets. GPTQ is a notable example that uses layer-wise quantization and updates weights with inverse Hessian information \citep{frantar2022gptq}. 

\paragrapharrow{Low-rank approximation for LLMs.}
Despite the popularity of low-rank factorization as a neural network compression technique, its application in LLMs remains relatively underexplored.
This gap is addressed through the introduction of novel low-rank decomposition methods designed specifically for LLMs, such as ASVD \citep{yuan2023asvd}.
To be more specific, standard SVD focuses on compressing original weight matrices without considering parameter importance, which can result in larger compression errors. To address this, FWSVD incorporates Fisher information to weigh parameter importance,  ASVD scales the weight matrix by a diagonal matrix representing input channel impact to account for activation distribution effects, while SVD-LLM further scales the weight matrix by a Cholesky decomposition of the activation matrix to upper bound the reconstruction error \citep{hsu2022language, yuan2023asvd, wang2024svd}.
However, all of these methods face challenges in balancing compression efficiency and accuracy, especially under medium to high compression scenarios for datasets with different languages or  tasks.

Despite advancements such as the introduction of rank-adaptive methods like FWSVD, ASVD, and SVD-LLM, which aim to optimize compression efficiency and accuracy by considering parameter importance and activation distribution effects, significant challenges remain. These methods struggle with severe accuracy degradation under medium to high compression ratios due to  unawareness of the discrepancy between different datasets.
In summary, while substantial progress has been made in compressing LLMs, there remains a need for more efficient and accessible deployment methods that balance compression effectiveness with minimal loss in model accuracy.

\section{Method}
In this section, we  outline the formulation of the proposed nested decomposition approach  using matrix decomposition. 
We begin by discussing the theory of low-rank approximation and the  source of activation-aware low-rank approximation,  followed by introducing our proposed method of nested decomposition aimed at achieving low-rank approximations for LLMs.

To approximate a matrix $\bA\in \real^{m\times n}$ of rank $r$ with a lower rank-$k$ matrix $\widetildebA$ ($k<r$), the approximation can be evaluated using the Frobenius norm:
\begin{equation}
	\widetildebA = \mathop{\arg\min}_{\rank(\bB)=k} \ \ \normf{\bA - \bB},
\end{equation}
where $\rank(\cdot)$ denotes the rank of a matrix.
The \text{Frobenius norm} of a matrix $\bA\in \real^{m\times n}$ is defined as 
\begin{equation}\label{definition:frobernius-in-svd}
\normf{\bA} = \sqrt{\sum_{i=1,j=1}^{m,n} (a_{ij})^2}=\sqrt{\trace(\bA\bA^\top)}=\sqrt{\trace(\bA^\top\bA)} = \sqrt{\sigma_1^2+\sigma_2^2+\ldots+\sigma_r^2}, 
\end{equation}
where $\sigma_1, \sigma_2, \ldots, \sigma_r$ are the nonzero singular values of $\bA$, and $\trace(\cdot)$ denotes the trace of the underlying matrix.
Then, we can recover the optimal rank-$k$ approximation through  the following theorem.
\begin{theorem}[Eckart-Young-Mirsky Theorem]\label{theorem:young-theorem-spectral}
Given a matrix $\bA\in \real^{m\times n}$, $1\leq k\leq \rank(\bA)=r$, and let $\bA_k$ be the truncated SVD (TSVD) of $\bA$ retaining the largest $k$ singular terms, i.e., $\bA_k = \sum_{i=1}^{k} \sigma_i\bu_i\bv_i^\top$ from the  SVD of $\bA=\sum_{i=1}^{r} \sigma_i\bu_i\bv_i^\top$ by zeroing out the $r-k$ trailing singular values of $\bA$. Then, $\bA_k$ is the optimal rank-$k$ approximation to $\bA$ in terms of the Frobenius norm. {Note that $\bA_k$ can be stored using $(m+n)k + k$ entries (or $(m+n)k$ entries if we absorb the singular values into the other two decomposed matrices), significantly reducing storage requirements compared to $mn$ entries required for the full matrix.} 
\end{theorem}

Moreover, it can also be shown that  $\bA_k$ is  the optimal rank-$k$ approximation to $\bA$ in terms of the spectral norm or any other unitarily invariant matrix norms \citep{lu2022matrix}. The minimal error is given by the Euclidean norm of the discarded singular values: $\normf{\bA-\bA_k} =\sqrt{\sigma_{k+1}^2 +\sigma_{k+2}^2+\ldots +\sigma_{r}^2}$.

\paragrapharrow{Activation-aware low-rank approximation.}
In many applications,   low-rank approximation is applied to weight matrices within transformer structures for LLMs \citep{vaswani2017attention}.
However, due to the high variance in input activations, simply applying standard SVD for such structure compression can lead to significant accuracy degradation.

To address this issue, {activation-aware low-rank approximation} of matrix $\bA$ solves the following optimization problem instead:
\begin{equation}\label{equation:awaretruc_loss}
\widetildebW\widetildebZ\equiv \widetildebA = \mathop{\arg\min}_{\rank(\bB)=k}\;\normf{(\bA -\bB)\bX},
\end{equation}
where $\bX\in\real^{n\times p}$ represents the input activation matrix, and $\bA,\bB\in\real^{m\times n}$ ($\bA$ is typically referred to as a weight matrix in neural network or transformer structures).

Specifically, the activation-aware low-rank approximation approach extracts a  matrix $ \bS $ from $ \bX $; for example, $\bS$ could be a diagonal matrix, with each element in the diagonal  chosen as the absolute mean value of each dimension (i.e., $s_{ii} $ is the absolute mean value of the corresponding dimension $\bx_i$ for each $i$). 
This matrix $ \bS $ is then used to normalize $ \bX $, transforming $ \bA $ into $ (\bA \bS)(\bS^{-1} \bX) $. Subsequently, SVD is performed on $ \bA \bS $ to obtain the decomposed matrices:
\begin{equation}\label{equation:asvd0}
\text{(ASVD-0):}\qquad \widetildebA_k = \mathop{\arg\min}_{\rank(\bB)=k}\;\normf{\bA \bS - \bB}.
\end{equation}
This compression approach for LLMs is called \textit{ASVD-0} \citep{yuan2023asvd}.

More formally, the following theorem provides a sub-optimal solution to the activation-aware low-rank approximation problem \eqref{equation:awaretruc_loss}, where the squared loss is upper bounded by the sum of squared singular values of a specific matrix related to the weight matrix and the activation.
\begin{theorem}[ASVD-I: Activation-Aware Low-Rank Approximation by Cholesky]\label{theorem:asvd1}
Given a matrix $\bA\in\real^{m\times n}$,
let  $\bA\bS=\sum_{i=1} \sigma_i\bu_i\bv_i^\top$ be the SVD of $\bA\bS$, where $\bS\bS^\top$ is the Cholesky decomposition of $\bX\bX^\top$, and  $\bX$ has full rank.
Then, 
\begin{enumerate}
\item Truncating  the $j$-th singular value of $\bA\bS$ to obtain $\widetildebA\bS = \sum_{i\neq j} \sigma_i\bu_i\bv_i^\top$, the compression loss $ \ell_j =\normfbig{(\bA -\widetildebA)\bX} $ equals  $ \sigma_j $.
\item  Truncating   the smallest singular values of $\bA\bS$ to obtain $\widetildebA\bS = \sum_{i=1}^k \sigma_i\bu_i\bv_i^\top$, the squared compression loss is $\ell^2= \sum_{i=k+1}^{r} \sigma_i^2$, where $r$ is the rank of $\bA\bS$.
\end{enumerate}
\end{theorem}
\begin{proof}
\textbf{(1).}
Since the  matrix $ \bS $ is the Cholesky decomposition of $ \bX\bX^\top $, we have $ \bS\bS^\top = \bX\bX^\top $. Therefore,
$$
\small
\begin{aligned}
\ell_j 
&=\normfbig{(\bA -\widetildebA)\bX}
=
\normfbig{\sigma_j \bu_j \bv_j ^\top\bS^{-1}\bX}
=
\sigma_i \trace(\bu_j \bv_j ^\top \bS^{-1}\bX\bX^\top \bS^{-\top}\bv_j \bu_j ^\top)^{1/2}
= \sigma_j .
\end{aligned}
$$
Thus, the compression loss $ \ell_j $ from truncating $ \sigma_j  $ equals the singular value $ \sigma_j  $ itself.
\paragraph{(2).}
If we truncate $ \sigma_{k+1}, \sigma_{k+2},  \ldots, \sigma_r $, where $r$ is the rank of $\bA\bS$, the square of the loss $ \ell $ is:
\begin{align*}
\ell^2 &=\normf{\sum_{i=k+1}^{r} \sigma_i \bu_i \bv_i^\top \bS^{-1} \bX}^2 = \sum_{j=k+1}^{r} \sum_{i=k+1}^{r} \sigma_i \sigma_j \trace(\bu_i \bv_i^\top \bS^{-1} \bX\bX^\top \bS^{-\top} \bv_j \bu_j^\top) \\
&= \sum_{i=k+1}^{r} \sigma_i^2 \trace(\bu_i \bv_i^\top \bS^{-1} \bX\bX^\top \bS^{-\top} \bv_i \bu_i^\top) = \sum_{i=k+1}^{r} \ell_i^2 = \sum_{i=k+1}^{r} \sigma_i^2.
\end{align*}
This completes the proof.
\end{proof}

\paragrapharrow{Activation-aware low-rank approximation using SVD.}
Theorem~\ref{theorem:asvd1} demonstrates that truncating the smallest singular values of $\bA\bS$ provides a sub-optimal solution for problem \eqref{equation:awaretruc_loss}. Nevertheless, this sub-optimal solution is not unique.
Further observation indicates  that the Cholesky decomposition of $\bX\bX^\top$ can be replaced by its  SVD (or spectral decomposition with nonnegative eigenvalues since $\bX\bX^\top$ is positive semidefinite).
\begin{theorem}[ASVD-II: Activation-Aware Low-Rank Approximation by SVD]\label{theorem:asvd2}
Given a matrix $\bA\in\real^{m\times n}$,
let  $\bA\bS=\sum_{i=1} \sigma_i\bu_i\bv_i^\top$ be the SVD of $\bA\bS$, where $\bS\bS^\top\triangleq (\bP\bLambda^{1/2}) (\bLambda^{1/2} \bP^\top) =  \bP\bLambda\bP^\top$ is the SVD of $\bX\bX^\top$,   and  $\bX$ has full rank.
Then, 
\begin{itemize}
\item[(i)] The results (1) and (2) from Theorem~\ref{theorem:asvd1} also hold.
\item[(ii)] The results of activation-aware low-rank approximation using Cholesky decomposition in Theorem~\ref{theorem:asvd1}  and using SVD are equivalent.
\end{itemize}
\end{theorem}
\begin{proof}
The first part follows similarly from Theorem~\ref{theorem:asvd1}. For \textbf{(ii)},
the matrix $\bP\bLambda^{1/2}$ admits the LQ decomposition $\bP\bLambda^{1/2} = \bL\bQ$~\footnote{The proof of existence can be found in \citet{lu2022matrix}.} such that $\bS\bS^\top = \bL\bL^\top$, where $\bS$ and $\bL$ are both lower triangular, and $\bQ$ is orthogonal.
This proves the equivalence by the uniqueness of the Cholesky decomposition \citep{lu2022matrix}.
\end{proof}

Theorem~\ref{theorem:asvd2} establishes the equivalence between the Cholesky decomposition and SVD for activation-aware low-rank approximation when the activation matrix $\bX$ has full rank. 
However, in practice, computing the eigenvalues of $\bX\bX^\top$ becomes necessary to adjust for positive semidefiniteness if any eigenvalues are zero when using the Cholesky decomposition. 
This: 1) increases the computational complexity because it requires both the Cholesky decomposition and the computation of eigenvalues;
2) introduces biases when some eigenvalues are zero.
In contrast, the method via SVD does not require adjustments for zero eigenvalues since pseudo-inverses can be applied when this happens and avoids additional computations. Moreover, the SVD approach offers an intuitive interpretation: $(\bA\bS)^\top=\bLambda^{1/2}\bP^\top\bA^\top$ initially rotates or reflects the weight matrix using the orthogonal matrix $\bP^\top$, and then scales each dimension by the square roots of the singular values (or the eigenvalues since $\bX\bX^\top$ is symmetric); this facilitates more nuanced feature selection and analysis in future work (see Theorem~\ref{theorem:asvd3} for more discussions).
We refer to the method in Theorem~\ref{theorem:asvd1} as \textit{ASVD-I}, which is known as SVD-LLM in the original paper \citep{wang2024svd}, and the method in Theorem~\ref{theorem:asvd2} as \textit{ASVD-II}.

\paragrapharrow{Nested low-rank approximation.}
We aware that one source of inaccuracy in activation-aware low-rank approximation methods for LLM compression arises from the compression loss incurred by using a calibration dataset to obtain the activation. When we use a dataset of activation $\bX_1$ to obtain the compressed weight, it may lose information relevant to other datasets, thereby introducing bias and performance degradation.
The nested decomposition approach addresses this issue through the following steps:
\begin{subequations}\label{equation:nesterd_lowapp}
\begin{align}
\widetildebW_1\widetildebZ_1 \equiv \widetildebA_1 &=  \mathop{\arg\min}_{\rank(\bB)=k_1}\;\normf{(\bA -\bB)\bX},
&\qquad& \rank(\widetildebW_1)=\rank(\widetildebZ_1) = k_1;
\label{equation:nesterd_lowapp1} \\
\widetildebW_2\widetildebZ_2 \equiv \widetildebA_2 &=  \mathop{\arg\min}_{\rank(\bB)=k_2}\;\normf{\bB -  (\bA - \widetildebA_1)},
&\qquad&  \rank(\widetildebW_2)=\rank(\widetildebZ_2) = k_2,
\label{equation:nesterd_lowapp2}
\end{align}
\end{subequations}
where \eqref{equation:nesterd_lowapp1} is applied to reduce the variance due to input activations (same as the ASVD approaches), and \eqref{equation:nesterd_lowapp2} aims to approximate the original weight matrix $\bA$ as closely as possible.
By setting $k_1+k_2 = k$, the same level of compression ratio as in \eqref{equation:awaretruc_loss} is achieved.
The activation when receiving a new input activation $\widetildebX$ is then given by 
\begin{equation}\label{equation:nest_out}
\bO=\widetildebA_1\widetildebX + \widetildebA_2\widetildebX
= \widetildebW_1(\widetildebZ_1\widetildebX) + \widetildebW_2(\widetildebZ_2\widetildebX)
.
\end{equation}
The complexity of \eqref{equation:nest_out} is $\mathcal{O}(2n(p+m)(k_1+k_2))$ floating-points operations (flops), which matches that of the ASVD case.

Although Theorems~\ref{theorem:asvd1} and \ref{theorem:asvd2} show that Step~\eqref{equation:nesterd_lowapp1} is upper bounded by the Frobenius norm of the truncating matrix from $\bA\bS$, where $\bS\bS^\top\equiv \bP\bLambda\bP^\top$ is either the Cholesky decomposition or  SVD of $\bX\bX^\top$, Step~\eqref{equation:nesterd_lowapp2} can be more  general and achieved via SVD, leading to the terms \textit{NSVD-I}  or \textit{NSVD-II} (following Theorem~\ref{theorem:asvd1} or Theorem~\ref{theorem:asvd2}, respectively, for Step~\eqref{equation:nesterd_lowapp1}). 
Alternatively, \eqref{equation:nesterd_lowapp2} can be performed  using low-rank interpolative decomposition (ID) \citep{martinsson2011randomized, lu2022bayesian}; hence the names \textit{NID-I} or \textit{NID-II}.

\paragrapharrow{Other failure trials.}

During the development of this work, we also experimented with other variants of the ASVD approach. Although these variants did not show better performance in our experiments, they may provide insights for future research.
\begin{theorem}[ASVD-III: Activation-Aware Low-Rank Approximation via Scaling Eigenvalues]\label{theorem:asvd3}
Consider the same setting as Theorem~\ref{theorem:asvd2}.
Let  $\bA\bP\cdot ({\gamma}\bI) =\sum_{i=1} \widehat{\sigma}_i\widehat{\bu}_i\widehat{\bv}_i^\top$ be the SVD of $\bA\bP \cdot (\gamma\bI)$, where $  \bP\bLambda\bP^\top$ is the SVD of $\bX\bX^\top$,   and  $\bX$ has full rank. Then, 
\begin{itemize}
\item[(a)]  Truncating  the $j$-th singular value of $\bA\bP\cdot (\gamma\bI)$ to obtain $\widetildebA\bP\cdot (\gamma\bI) = \sum_{i\neq j} \widehat{\sigma}_i\widehat{\bu}_i\widehat{\bv}_i^\top$, the compression loss $ \widehat{\ell}_j =\normfbig{(\bA -\widetildebA)\bX} $ equals  $ \widehat{\sigma}_j \trace(\frac{1}{\gamma^2}\bLambda \widehat{\bv}_i\widehat{\bv}_i^\top) $.
\item[(b)]   Truncating   the smallest singular values of $\bA\bP \cdot (\gamma\bI)$ to obtain $\widetildebA\bP \cdot (\gamma\bI) = \sum_{i=1}^k \widehat{\sigma}_i\widehat{\bu}_i\widehat{\bv}_i^\top$, the squared compression loss is $\widehat{\ell}^2= \sum_{i=k+1}^{r} \widehat{\sigma}_i^2 \trace(\frac{1}{\gamma^2}\bLambda \widehat{\bv}_i\widehat{\bv}_i^\top)$, where $r$ is the rank of $\bA\bP$.
\end{itemize}

\end{theorem}
\begin{proof}
\textbf{(a).}
It follows that 
$$
\scriptsize
\begin{aligned}
\widehat{\ell}_j 
&=\normf{(\bA\bP{\gamma}\bI -\widetildebA\bP{\gamma}\bI)\frac{\bP^\top}{\gamma}\bX}
=
\normf{\widehat{\sigma}_j \widehat{\bu}_j \widehat{\bv}_j ^\top\frac{\bP^\top}{\gamma}\bX}
=
\widehat{\sigma}_j  \trace\left(\widehat{\bu}_j \widehat{\bv}_j ^\top \frac{\bP^\top}{\gamma}\bX\bX^\top \frac{\bP}{\gamma}\widehat{\bv}_j \widehat{\bu}_j ^\top\right)^{1/2}
= \widehat{\sigma}_j  \trace\left(\frac{\bLambda}{\gamma^2} \widehat{\bv}_j \widehat{\bv}_j ^\top\right).
\end{aligned}
$$

\textbf{(b).}
If we truncate $ \widehat{\sigma}_{k+1}, \widehat{\sigma}_{k+2},  \ldots, \widehat{\sigma}_r $, where $r$ is the rank of $\bA\bP$, the square of the loss $ \widehat{\ell} $ is:
\begin{align*}
\widehat{\ell}^2 &=\normf{\sum_{i=k+1}^{r} \widehat{\sigma}_i \widehat{\bu}_i \widehat{\bv}_i^\top \frac{1}{\gamma}\bP^\top \bX}^2 = \sum_{j=k+1}^{r} \sum_{i=k+1}^{r} \widehat{\sigma}_i \widehat{\sigma}_j \trace(\widehat{\bu}_i \widehat{\bv}_i^\top \frac{1}{\gamma}\bP^\top \bX\bX^\top \bP \frac{1}{\gamma}\widehat{\bv}_j \widehat{\bu}_j^\top) \\
&= \sum_{i=k+1}^{r} \widehat{\sigma}_i^2 \trace(\widehat{\bu}_i \widehat{\bv}_i^\top\frac{1}{\gamma}\bP^\top \bX\bX^\top \bP\frac{1}{\gamma} \widehat{\bv}_i \widehat{\bu}_i^\top) = \sum_{i=k+1}^{r} \widehat{\ell}_i^2 = \sum_{i=k+1}^{r} \widehat{\sigma}_i^2 \trace(\frac{1}{\gamma^2}\bLambda \widehat{\bv}_i\widehat{\bv}_i^\top)^2.
\end{align*}
This completes the proof.
\end{proof}
Theorem~\ref{theorem:asvd3} demonstrates that if we choose $\gamma \triangleq \max(\bLambda^{1/2})$, then $\trace(\frac{1}{\gamma^2}\bLambda \widehat{\bv}_i\widehat{\bv}_i^\top)$ is upper bounded by one, and thus the squared loss is upper bounded by the sum of squared singular values. 
We refer to this method as \textit{ASVD-III}.
However, the singular values in Theorem~\ref{theorem:asvd2} come from $\bA\bP\bLambda^{1/2}$, while those in Theorem~\ref{theorem:asvd3} come from $\bA\bP\cdot (\gamma\bI)$. ASVD-III then benefits if the singular values of $\bA\bP\bLambda^{1/2}$ are more clustered, indicating a small hierarchical relationship among the singular values.
However, in our experiments,  we did not observe any improvement using ASVD-III; this is partly because the singular values in $\bA\bP\bLambda^{1/2}$ already exhibit a hierarchical relationship, which underscores why SVD-based low-rank approximation methods can perform well.

\section{Experiments}

We compare NSVD approaches against state-of-the-art baselines including  standard SVD, ASVD-0, and ASVD-I.
To demonstrate the versatility  of our method, following \citet{wang2024svd},  we evaluate the performance of NSVD alongside these baselines across \textcolor{black}{six} models from three distinct  LLM families at various scales: LLaMA-7B, 13B, \textcolor{black}{30B}, \citep{touvron2023llama}, OPT-6.7B \citep{zhang2022opt}, Vicuna-7B \citep{chiang2023vicuna} and Mistral-7B \citep{chaplot2023albert}; and \textcolor{black}{eight} language modeling datasets: WikiText-2, PTB, C4, SNIPS, AlpacaEval, MCTest, CMRC (CN), and AlpacaEval (JP) \citep{merity2016pointer, marcus1993building, raffel2020exploring, coucke2018snips, dubois2023alpacafarm, richardson2013mctest, cui2018span, junsashihara}. The first six datasets consist of \textbf{English} sentences, while CMRC (CN)  contains \textbf{Chinese} sentences, and AlpacaEval (JP)  includes texts in \textbf{Japanese}.

Following   \citet{yuan2023asvd} and \citet{wang2024svd}, we randomly select 256 samples from the training set of  WikiText-2 to serve as the calibration data.
In all scenarios, we use the test split of each dataset for evaluation purposes; if a test set is unavailable, we substitute it with the training set. It is anticipated that the CMRC (CN) and AlpacaEval (JP) datasets could yield markedly different activations due to their respective languages. Consequently, these two datasets are crucial for assessing the ``overfitting" potential during the post-training procedure when applying low-rank approximation methods to LLMs.

\newcommand{\toppp}{\multicolumn{1}{c|}{}}
\newcommand{\topaa}{\multicolumn{1}{c|}{\multirow{4}{*}{10\%}}}
\newcommand{\topbb}{\multicolumn{1}{c|}{\multirow{4}{*}{20\%}}}
\newcommand{\topcc}{\multicolumn{1}{c|}{\multirow{4}{*}{30\%}}}
\newcommand{\topdd}{\multicolumn{1}{c|}{\multirow{4}{*}{40\%}}}
\newcommand{\topee}{\multicolumn{1}{c|}{\multirow{4}{*}{50\%}}}
\newcommand{\nestdecom}{\colorbox{gray!10}{\text{NSVD-I}}}
\newcommand{\nestdecomii}{\colorbox{gray!10}{\text{NSVD-II}}}
\newcommand{\nidi}{\colorbox{gray!10}{\text{NID-I}}}
\newcommand{\asvdii}{\colorbox{gray!10}{\text{ASVD-II}}}
\newcommand\tabalign[1]{\multicolumn{1}{c|}{#1}}

\noindent
\begin{table}[h]
\centering
\resizebox{1.\textwidth}{!}{
\begin{tabular}{ccc|ccccccc|cc}
\midrule
\tabalign{\textsc{Ratio}}      & \tabalign{\textsc{Method}}    & WikiText-2 & PTB & C4   &   {SNIPS}   & {AlpacaEval}   & {MCTest} 
& {CMRC (CN)} & {AlpacaEval (JP)}   & Avg. Impro.
\\ \midrule
\tabalign{\textcolor{gray}{0\%}}  & \tabalign{\textcolor{gray}{Original}} & \textcolor{gray}{5.68}   &  \textcolor{gray}{8.35}     & \textcolor{gray}{7.34}   
& \textcolor{gray}{5.88}  & \textcolor{gray}{3.39}  & \textcolor{gray}{1.97}    & \textcolor{gray}{5.77}    & \textcolor{gray}{5.36}   & --
\\ \midrule
\topaa  & \tabalign{SVD}     & 2778.92 & 4011.84 & 4062.66 & 2135.73 & 2318.94 & 3163.11 & 37554.17 & 4894.34 & --    \\
\toppp  & \tabalign{ASVD-0}  & 26.95 & 41.91 & 37.65 & 18.82 & 11.59 & 3.41 & 85.55 & 123.64 & --  \\
\toppp  & \tabalign{ASVD-I}  & 7.07  & 12.42 & 12.04 & 7.84 & 4.81 & 2.21  & 14.05 & 15.93& -- \\
\cmidrule{2-10} 
\toppp  & \tabalign{\asvdii}  & 7.07  & 12.41 & 12.03 & 7.84 & 4.81 & 2.21 & 14.05 & 15.93 & --  \\
\toppp  & \tabalign{\nestdecom} & {7.08} (\textcolor{mylightbluetext}{$\uparrow$0.1\%})  & \textbf{12.03} (\textcolor{mylightbluetext}{$\downarrow$3.1\%})& \textbf{11.58} (\textcolor{mylightbluetext}{$\downarrow$3.8\%})    & \textbf{7.74}  (\textcolor{mylightbluetext}{$\downarrow$1.3\%}) 
& \textbf{4.66}  (\textcolor{mylightbluetext}{$\downarrow$3.1\%})  & \textbf{2.19}  (\textcolor{mylightbluetext}{$\downarrow$0.9\%}) &
\textbf{12.54}  (\textcolor{mylightbluetext}{$\downarrow$10.7\%})  & \textbf{14.89}  (\textcolor{mylightbluetext}{$\downarrow$6.5\%}) & 4.2\% \\
\toppp  & \tabalign{\nestdecomii} & {7.08} (\textcolor{mylightbluetext}{$\uparrow$0.1\%})  & \textbf{12.03} (\textcolor{mylightbluetext}{$\downarrow$3.1\%})& \textbf{11.58} (\textcolor{mylightbluetext}{$\downarrow$3.8\%})    & \textbf{7.74}  (\textcolor{mylightbluetext}{$\downarrow$1.3\%}) 
& \textbf{4.67}  (\textcolor{mylightbluetext}{$\downarrow$2.9\%})  & \textbf{2.19}  (\textcolor{mylightbluetext}{$\downarrow$0.9\%}) &
\textbf{12.54}  (\textcolor{mylightbluetext}{$\downarrow$10.7\%})  & \textbf{14.89}  (\textcolor{mylightbluetext}{$\downarrow$6.5\%}) & 4.2\%
\\ \midrule
\topbb  & \tabalign{SVD}    & 20083.00  & 20436.32  & 18827.16 & 39417.77 & 12539.73 & 18853.27 & 127786.81 & 123109.21 & --     \\
\toppp  & \tabalign{ASVD-0} & 91.85  & 133.18  & 129.77 & 54.27 & 50.18 & 8.66  & 721.09 & 908.54 & --      \\
\toppp  & \tabalign{ASVD-I} & 7.89   & 16.54  & 15.93  & 10.50 & 6.14 & 2.44 & 48.09 & 38.31 & --\\
\cmidrule{2-10}   
\toppp  & \tabalign{\asvdii} & 7.89   & 16.49  & 15.95  & 10.51 & 6.15 & 2.44 & 48.13 & 38.32 & --\\
\toppp  & \tabalign{\nestdecom} & {7.92} (\textcolor{mylightbluetext}{$\uparrow$0.4\%})  & \textbf{15.44} (\textcolor{mylightbluetext}{$\downarrow$6.7\%})& \textbf{15.23} (\textcolor{mylightbluetext}{$\downarrow$4.4\%})   & \textbf{9.93}  (\textcolor{mylightbluetext}{$\downarrow$5.4\%})
& \textbf{5.85}  (\textcolor{mylightbluetext}{$\downarrow$4.7\%})   & \textbf{2.40}  (\textcolor{mylightbluetext}{$\downarrow$1.6\%}) 
&  \textbf{34.49}  (\textcolor{mylightbluetext}{$\downarrow$28.3\%}) & \textbf{34.51}  (\textcolor{mylightbluetext}{$\downarrow$9.9\%})& 8.7\%\\
\toppp  & \tabalign{\nestdecomii} & {7.93} (\textcolor{mylightbluetext}{$\uparrow$0.5\%})  & \textbf{15.41} (\textcolor{mylightbluetext}{$\downarrow$6.8\%})& \textbf{15.27} (\textcolor{mylightbluetext}{$\downarrow$4.1\%})     & \textbf{9.95}  (\textcolor{mylightbluetext}{$\downarrow$5.2\%})
&\textbf{5.86} (\textcolor{mylightbluetext}{$\downarrow$4.6\%})     & \textbf{2.40}  (\textcolor{mylightbluetext}{$\downarrow$1.6\%}) 
& \textbf{34.51}  (\textcolor{mylightbluetext}{$\downarrow$28.2\%}) & \textbf{34.51}  (\textcolor{mylightbluetext}{$\downarrow$9.9\%}) & 8.6\%
\\ \midrule
\topcc  & \tabalign{SVD}    & 13183.07  & 17321.18  & 21055.10 & 6899.56 & 16912.55 & 19886.18 & 203777.87 & 160611.69 & --   \\
\toppp  & \tabalign{ASVD-0} & 228.77 & 326.20 & 312.18 & 115.65 & 157.33 & 42.37 & 1406.88 & 1359.26 & --           \\
\toppp  & \tabalign{ASVD-I} & 9.51   & 27.11  & 25.81   & 17.89 & 9.45 & 3.06  & 688.23 & 813.23 & --\\
\cmidrule{2-10}   
\toppp  & \tabalign{\asvdii} & 9.51   & 27.15  & 25.81   & 17.95 & 9.45 & 3.06 & 689.93 & 814.98 & --\\
\toppp  & \tabalign{\nestdecom} & {9.64} (\textcolor{mylightbluetext}{$\uparrow$1.4\%})  & \textbf{25.19} (\textcolor{mylightbluetext}{$\downarrow$7.1\%})& \textbf{24.41} (\textcolor{mylightbluetext}{$\downarrow$5.4\%})    
& \textbf{15.74} (\textcolor{mylightbluetext}{$\downarrow$12.1\%})    
& \textbf{8.85} (\textcolor{mylightbluetext}{$\downarrow$6.3\%})    & \textbf{3.02} (\textcolor{mylightbluetext}{$\downarrow$1.3\%})  
& \textbf{577.63} (\textcolor{mylightbluetext}{$\downarrow$16.1\%}) & \textbf{367.69} (\textcolor{mylightbluetext}{$\downarrow$54.8\%})& 14.7\% \\
\toppp  & \tabalign{\nestdecomii} & {9.65} (\textcolor{mylightbluetext}{$\uparrow$1.5\%})  & \textbf{25.31} (\textcolor{mylightbluetext}{$\downarrow$6.6\%})
& \textbf{24.51} (\textcolor{mylightbluetext}{$\downarrow$5.0\%})   & \textbf{15.80} (\textcolor{mylightbluetext}{$\downarrow$11.7\%})    
& \textbf{8.88} (\textcolor{mylightbluetext}{$\downarrow$6.0\%})    & \textbf{3.02} (\textcolor{mylightbluetext}{$\downarrow$1.3\%})  
& \textbf{579.14} (\textcolor{mylightbluetext}{$\downarrow$15.9\%})  & \textbf{373.64} (\textcolor{mylightbluetext}{$\downarrow$54.1\%})  & 14.4\%
\\ \midrule
\topdd  & \tabalign{SVD}    & 52415.35  & 59743.13  & 47712.12 & 48902.68 & 34563.55 & 50406.63 & 138049.05 & 159592.52 & --   \\
\toppp  & \tabalign{ASVD-0} & 681.86 & 1255.66   & 829.90 & 351.98 & 434.39 & 294.58  & 9518.03 & 50113.53 & --       \\
\toppp  & \tabalign{ASVD-I} & 13.74  & 62.82  & 55.55   & 38.96 & 19.72 & 4.78   & 20629.84 & 35974.28 & -- \\ 
\cmidrule{2-10}  
\toppp  & \tabalign{\asvdii} & 13.74  & 62.88  & 55.47   & 38.84 & 19.83 & 4.80  & 20816.34 & 35994.52& --  \\  
\toppp  & \tabalign{\nestdecom} & {14.24} (\textcolor{mylightbluetext}{$\uparrow$3.6\%})  & \textbf{61.16} (\textcolor{mylightbluetext}{$\downarrow$2.6\%})
& \textbf{54.22} (\textcolor{mylightbluetext}{$\downarrow$2.4\%})  & \textbf{35.57} (\textcolor{mylightbluetext}{$\downarrow$8.7\%})
& \textbf{18.88} (\textcolor{mylightbluetext}{$\downarrow$4.3\%})  & \textbf{4.75} (\textcolor{mylightbluetext}{$\downarrow$0.6\%}) 
& \textbf{11899.94} (\textcolor{mylightbluetext}{$\downarrow$42.3\%}) & \textbf{11729.03} (\textcolor{mylightbluetext}{$\downarrow$67.4\%}) & 18.3\% \\
\toppp  & \tabalign{\nestdecomii} & {14.29} (\textcolor{mylightbluetext}{$\uparrow$4.0\%})  & \textbf{61.57} (\textcolor{mylightbluetext}{$\downarrow$2.0\%})& \textbf{54.18} (\textcolor{mylightbluetext}{$\downarrow$2.5\%})   
& \textbf{35.37} (\textcolor{mylightbluetext}{$\downarrow$9.2\%})
& \textbf{18.99} (\textcolor{mylightbluetext}{$\downarrow$3.7\%}) & \textbf{4.76} (\textcolor{mylightbluetext}{$\downarrow$0.4\%}) 
& \textbf{11917.10} (\textcolor{mylightbluetext}{$\downarrow$42.2\%}) & \textbf{11886.94} (\textcolor{mylightbluetext}{$\downarrow$67.0\%}) & 18.1\%
\\ \midrule
\topee  & \tabalign{SVD}    & 131872.88  & 86942.90  & 79661.03 & 52970.02 & 62037.18 & 100978.91 & 296638.73 & 607005.28 & --   \\
\toppp  & \tabalign{ASVD-0} & 1266.12 & 2448.33 & 1373.45 & 956.99 & 889.06 & 1029.50  & 142270.06 & 301097.83 & --     \\ 
\toppp  & \tabalign{ASVD-I} & 26.61   & 182.81 & 151.83   & 92.57 & 65.07 & 13.81 & 87327.64 & 210253.49  & --\\
\cmidrule{2-10}   
\toppp  & \tabalign{\asvdii} & 26.78   & 183.15 & 152.01   & 93.08 & 65.04 & 13.71 & 87538.71 & 210129.23 & --\\
\toppp  & \tabalign{\nestdecom} &{28.53} (\textcolor{mylightbluetext}{$\uparrow$7.2\%})  & \textbf{180.56} (\textcolor{mylightbluetext}{$\downarrow$1.2\%})& \textbf{146.62} (\textcolor{mylightbluetext}{$\downarrow$3.4\%}) & \textbf{89.91} (\textcolor{mylightbluetext}{$\downarrow$2.9\%}) 
& \textbf{62.81} (\textcolor{mylightbluetext}{$\downarrow$3.4\%}) & \textbf{13.11} (\textcolor{mylightbluetext}{$\downarrow$5.1\%}) 
& \textbf{60518.94} (\textcolor{mylightbluetext}{$\downarrow$30.7\%}) & \textbf{110191.48} (\textcolor{mylightbluetext}{$\downarrow$47.6\%})  & 13.5\% \\
\toppp  & \tabalign{\nestdecomii} &{28.73} (\textcolor{mylightbluetext}{$\uparrow$8.0\%})  & \textbf{180.60} (\textcolor{mylightbluetext}{$\downarrow$1.2\%})& \textbf{147.80} (\textcolor{mylightbluetext}{$\downarrow$2.7\%}) & \textbf{90.64} (\textcolor{mylightbluetext}{$\downarrow$2.1\%}) 
& \textbf{62.35} (\textcolor{mylightbluetext}{$\downarrow$4.2\%}) & \textbf{13.01} (\textcolor{mylightbluetext}{$\downarrow$5.8\%}) 
& \textbf{61064.03} (\textcolor{mylightbluetext}{$\downarrow$30.1\%}) & \textbf{110091.03} (\textcolor{mylightbluetext}{$\downarrow$47.6\%}) & 13.4\%
\\ \midrule
\end{tabular}
}
\captionof{table}{Zero-shot performance of LLaMA-7B compressed using NSVD-I and various baselines evaluated under compression ratios ranging from 10\% to 50\%, across \textcolor{black}{eight} language modeling datasets (measured by perplexity  (\textcolor{mylightbluetext}{$\downarrow$})). The best performance on each dataset is highlighted in bold. The relative performance gain compared to the best-performing baseline (here, ASVD-I) is indicated in blue text within bracket.
In all scenarios, we apply  NSVD-I or NSVD-II with $k_1=0.95\cdot k$, $k_2=k-k_1$, where $k$ is the low-rank parameter used in  ASVD approaches.
}
\label{tab:nsce_ppl1}
\end{table}

\subsection{Performance under different compression ratios} 
First, we evaluate the performance of LLaMA-7B compressed using NSVD (here, $k_1=0.95k$) and  baselines under compression ratios ranging from 10\% to 50\% across  all \textcolor{black}{eight} datasets. 
Our results include comparisons with ASVD-II, NSVD-I, and NSVD-II; since no improvements were observed using the proposed ASVD-III method, its results are omitted for brevity. Table~\ref{tab:nsce_ppl1} summarizes these findings. We observe that, as stated by Theorems~\ref{theorem:asvd1} and \ref{theorem:asvd2}, ASVD-I and ASVD-II yield equivalent performance when numerical errors are ignored. Similarly, NSVD-I and NSVD-II also produce comparable outcomes.

\begin{table}[H]
\resizebox{1.\textwidth}{!}{
\begin{tabular}{cccccccccc}
\midrule
\tabalign{\textsc{Similarity}}    & WikiText-2 & PTB & C4   &   {SNIPS}   & {AlpacaEval}   & {MCTest} 
& {CMRC (CN)} & {AlpacaEval (JP)}   & \\ \midrule
\tabalign{\textsc{Mean (std)}}    & \textcolor{black}{0.94 (0.04)} & \textcolor{black}{0.59 (0.17)} & \textcolor{black}{0.75 (0.19)}   &  \textcolor{black}{0.52 (0.23)}  & \textcolor{black}{0.72 (0.19)}   & \textcolor{black}{0.58 (0.22)}
& \textcolor{mylightbluetext}{0.46 (0.19)} & \textcolor{mylightbluetext}{0.45 (0.2)}  
\\ \midrule
\end{tabular}
}
\captionof{table}{The mean similarities and the standard deviations between the activations from the calibration set and those from other evaluation sets under LLaMA-7B. 
The calibration set shows an average similarity of less than 0.5 with both  CMRC (CN) or AlpacaEval (JP), while it has an average similarity of 0.94 with the test split of WikiText-2.
}
\label{tab:nsce_ppl0}
\end{table}
\begin{figure}[H]
\centering    
\vspace{-0.35cm} 
\subfigtopskip=2pt 
\subfigbottomskip=2pt 
\subfigcapskip=-5pt 
\subfigure[WikiText-2.]{\label{fig:exp1}%
\includegraphics[width=0.24\linewidth]{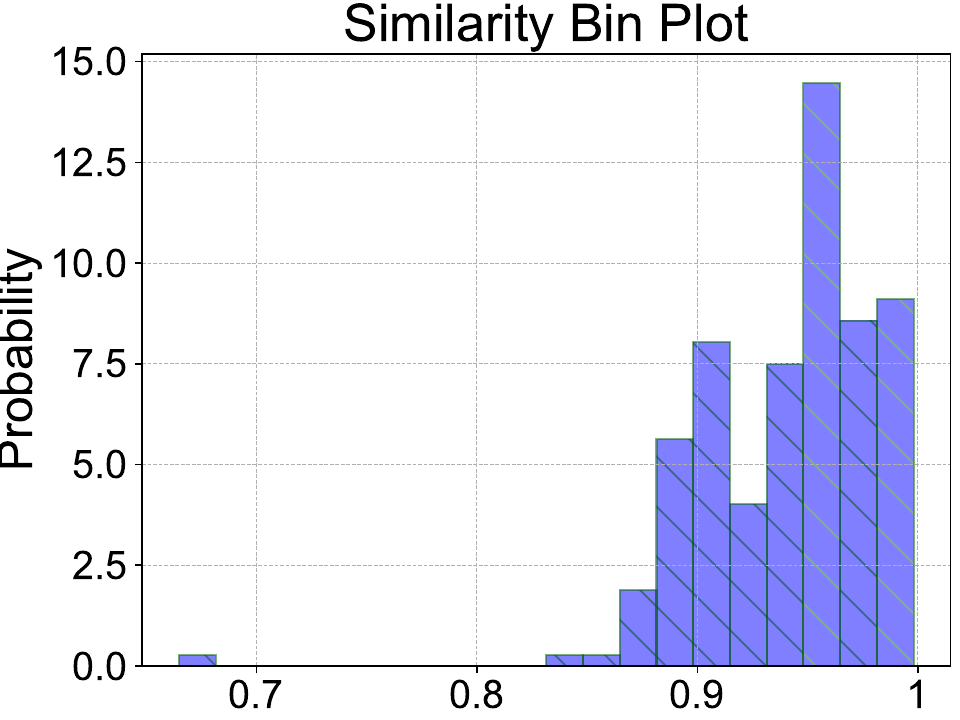}}%
\subfigure[PTB.]{\label{fig:exp2}%
\includegraphics[width=0.24\linewidth]{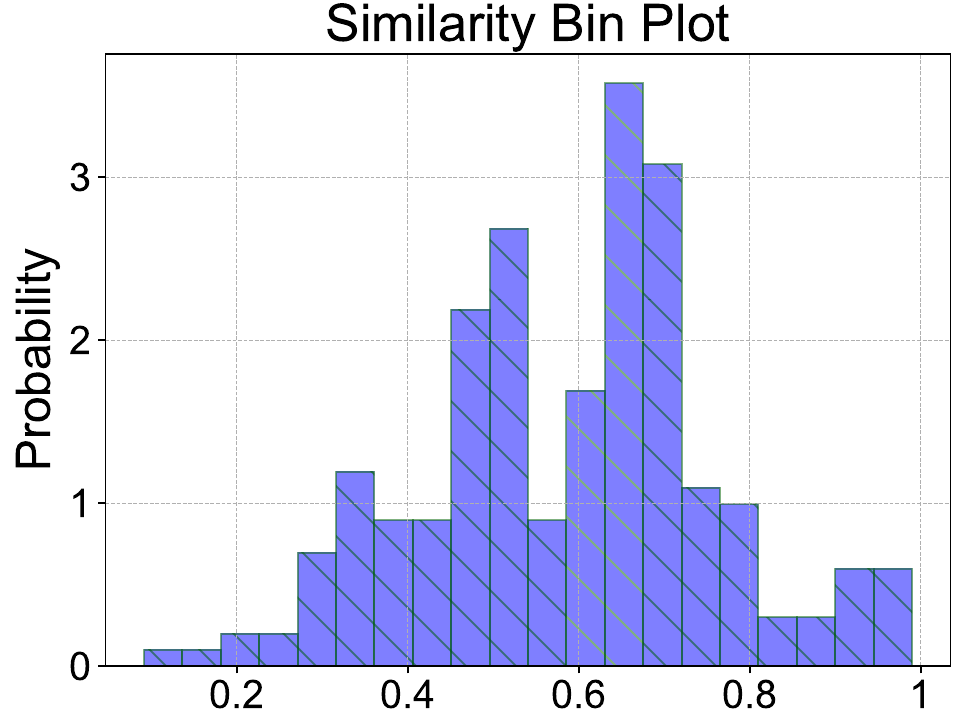}}%
\subfigure[C4.]{\label{fig:exp3}%
\includegraphics[width=0.24\linewidth]{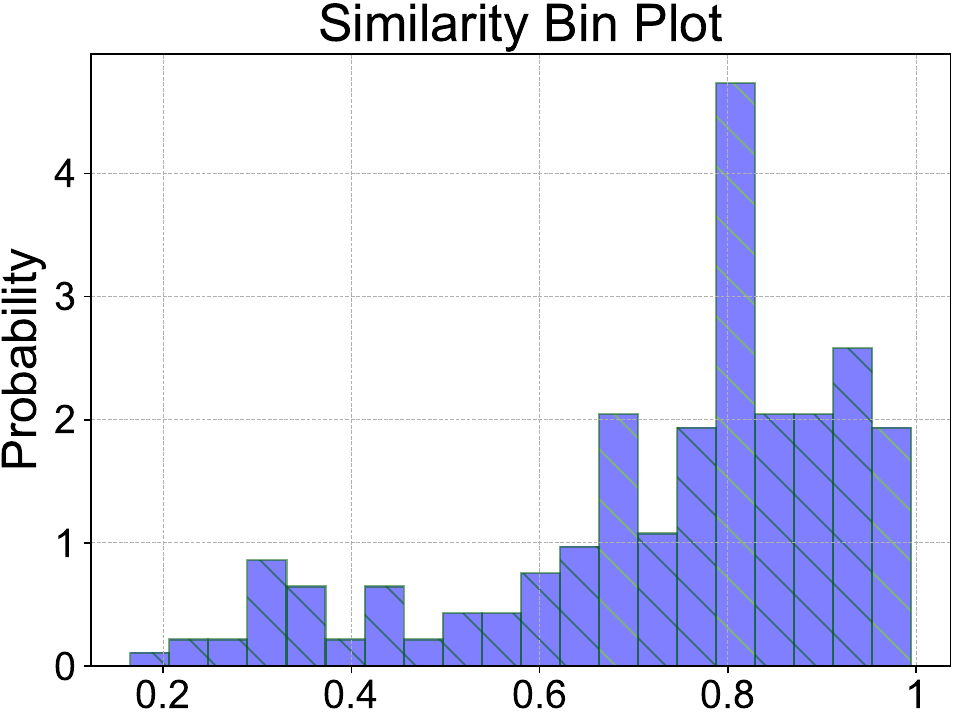}}
\subfigure[SNIPS.]{\label{fig:exp4}%
\includegraphics[width=0.24\linewidth]{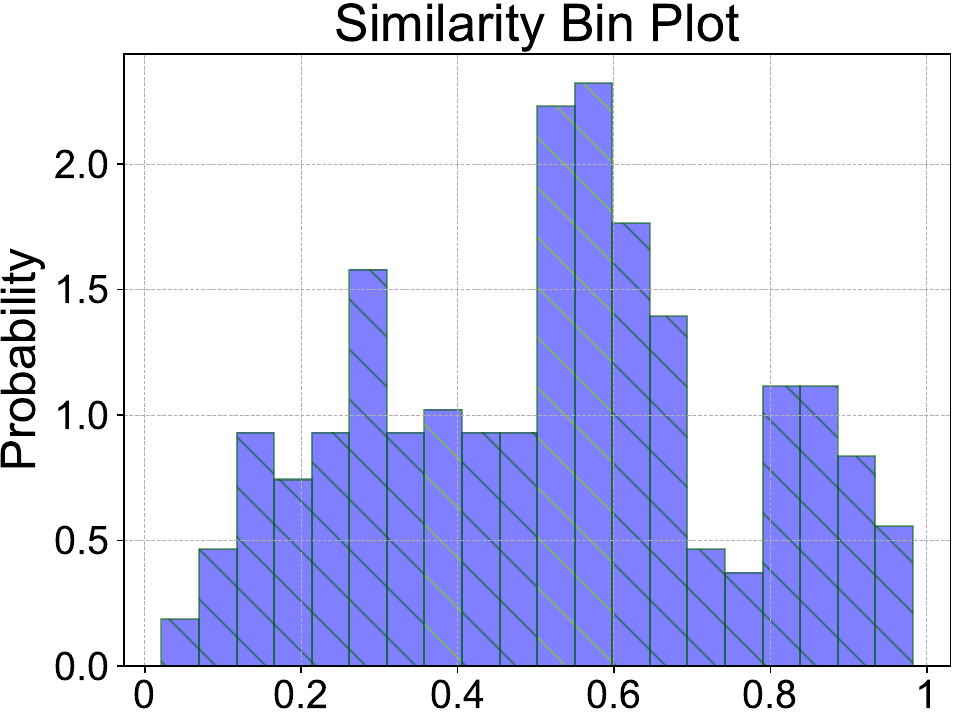}}\\
\subfigure[AlpacaEval.]{\label{fig:exp5}%
\includegraphics[width=0.24\linewidth]{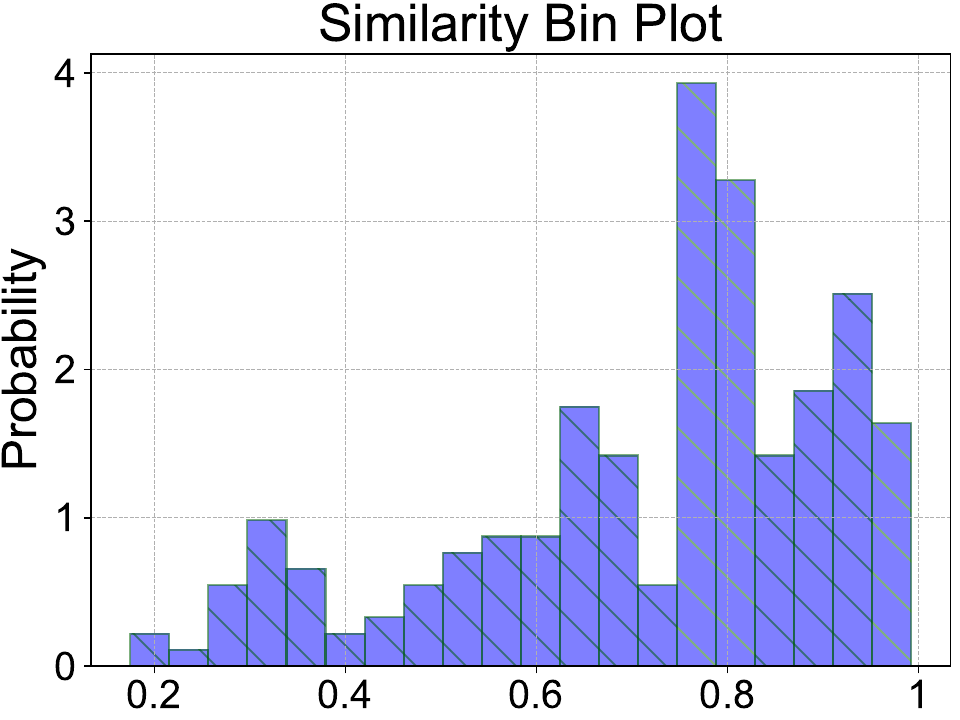}}%
\subfigure[MCTest.]{\label{fig:exp6}%
\includegraphics[width=0.24\linewidth]{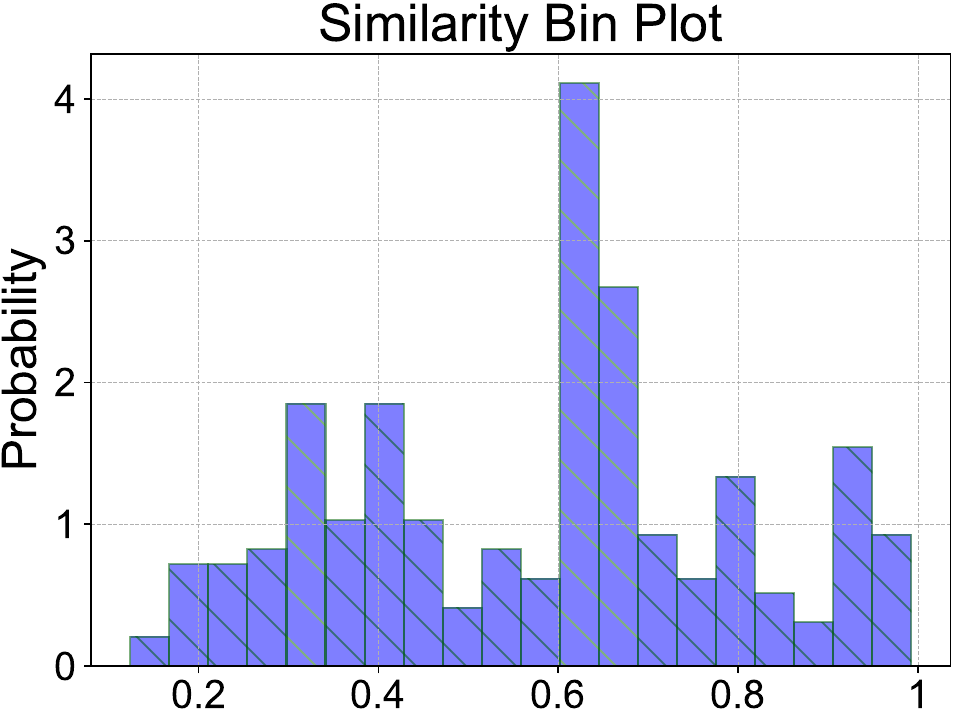}}
\subfigure[CMRC (CN).]{\label{fig:exp7}%
\includegraphics[width=0.24\linewidth]{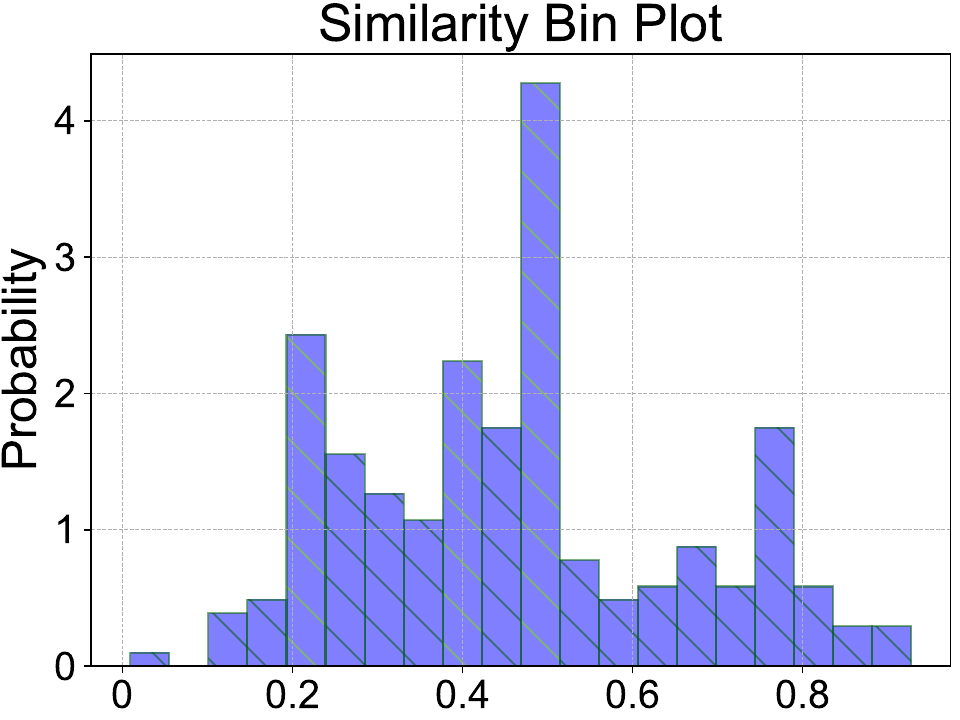}}%
\subfigure[AlpacaEval (JP).]{\label{fig:exp8}%
\includegraphics[width=0.24\linewidth]{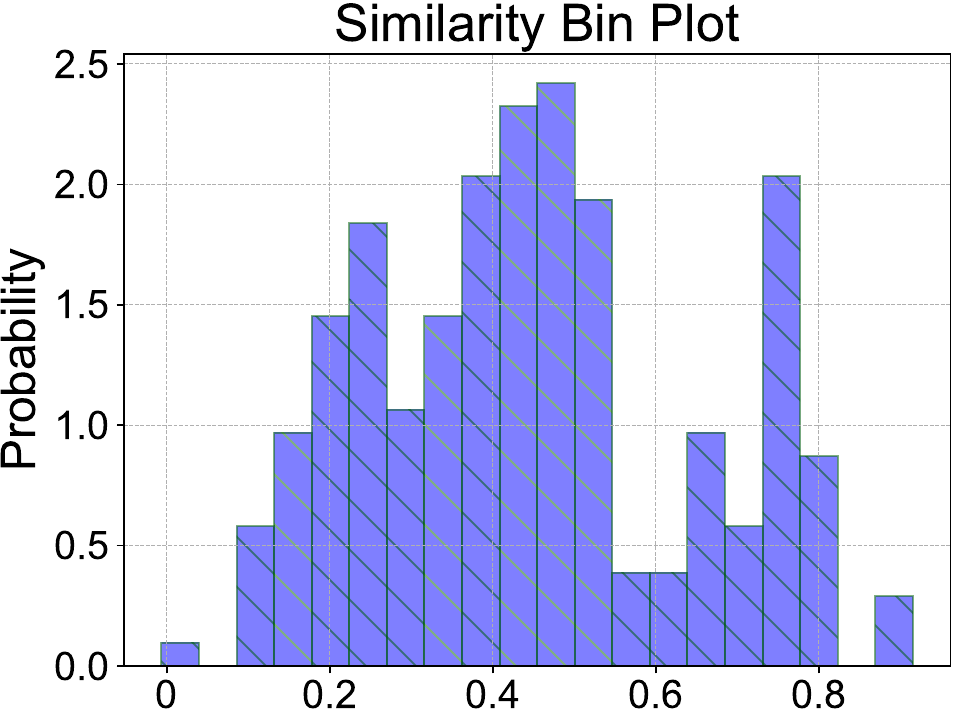}}%
\caption{The cosine similarity between the activations from the calibration dataset (the training set of WikiText-2) and those from various  evaluation datasets. Table~\ref{tab:nsce_ppl0}  presents the mean of similarities and the standard deviations. Specifically, the calibration set  shows an average similarity of less than 0.5 with both CMRC (CN) and AlpacaEval (JP).}
\label{fig:simi_dist}
\end{figure}

NSVD-I or NSVD-II consistently outperform standard SVD, ASVD-0, and ASVD-I across all the compression ratios. 
More importantly, NSVD exhibits significant advantages over  baselines under medium to high compression ratios. Specifically, at a 30\% compression ratio, compared to the best-performing baseline, NSVD-I reduces  perplexity on PTB, C4, SNIPS, AlpacaEval, MCTest, CMRC (CN), and AlpacaEval (JP) by 7.1\%, 5.4\%, 12.1\%, 6.3\%, 1.3\%, 16.1\%, and 54.8\%, respectively; 
when the compression ratio reaches 40\%, NSVD can reduce the perplexity by more than 60\%. 

Since the calibration set is derived from the training set of WikiText-2, NSVD does not improve performance on the test set of WikiText-2, which is expected since ASVD-I achieves a sub-optimal solution due to the similar activations using the the training and test sets of WikiText-2 (see Figure~\ref{fig:simi_dist} and Table~\ref{tab:nsce_ppl0}).
Therefore, we only provide the average improvement (Avg. Impro.) in Table~\ref{tab:nsce_ppl1} for the remaining \textcolor{black}{seven} evaluation sets.
On average, NSVD-I reduces  perplexity by 14.7\%, 18.3\%, and 13.5\% under 30\%, 40\%, and 50\% compression ratios, respectively.
These findings suggest that NSVD-I is highly effective for compressing LLMs, making them suitable for resource-constrained devices such as smartphones and embedded systems.

\textbf{Robustness.}
Figure~\ref{fig:simi_dist} illustrates the distribution of (cosine)  similarities for activations between the calibration set and each evaluation set under LLaMA-7B; while Table~\ref{tab:nsce_ppl0} presents the mean  similarities and their standard deviations. 
Notably, the activation under the test set of WikiText-2 has a similarity score of 0.94, indicating an extremely close match.
In contrast, the calibration set shows an average similarity of less than 0.5 with CMRC (CN) or AlpacaEval (JP); the proposed NSVD framework significantly outperforms baselines on  CMRC (CN) and AlpacaEval (JP).
Under a 40\% compression ratio, NSVD-I reduces  perplexity on CMRC (CN) and AlpacaEval (JP) by 42.3\% and 67.4\%, respectively.
These results, on the other hand, highlight  that NSVD is more robust in compressing LLMs for performing  out-of-sample tasks, e.g.,  the \textit{multilingual} settings or machine translation large language models.
Alternatively, if  LLMs are calibrated using a medical dataset, NSVD exhibits greater robustness when performing tasks like math competitions, demonstrating its effectiveness in multitask settings.

\subsection{Performance under different values of $k_1$} 
Given the flexibility of the NSVD framework in terms of the $k_1$ and $k_2$ parameters---so long as their sum equals that used in other low-rank approximation methods---we evaluate the performance of LLaMA-7B compressed using NSVD-I with varying values of $k_1$ ($k_1\in\{0.99k, 0.95k, 0.90k, 0.85k, 0.80k\}$) and compare it to ASVD-I under a  30\% compression ratio.
Table~\ref{tab:nsce_ppl2} summarizes the results for NSVD-I with different $k_1$ values. 
In most scenarios, the performance of NSVD-I initially  improves  and  then deteriorates as the value of $k_1$ decreases.
For both the CMRC (CN) and AlpacaEval (JP) datasets, the best improvement is observed when $k_1=0.80k$ for these experiments.
This suggests that for datasets with potentially very different activations, selecting a smaller $k_1$ is more effective.
A moderate approach for choosing $k_1$ would be within the range of $k_1=0.90k\sim 0.95k$.

\begin{table}[H]
\centering
\centering
\resizebox{1.\textwidth}{!}{
\begin{tabular}{ccc|ccccccc|cc}
\midrule
\tabalign{\textsc{$k_1$}}      & \tabalign{\textsc{Method}}    & WikiText-2\textcolor{mylightbluetext}{$\downarrow$} & PTB\textcolor{mylightbluetext}{$\downarrow$} & C4\textcolor{mylightbluetext}{$\downarrow$}   &   {SNIPS\textcolor{mylightbluetext}{$\downarrow$}}   & {AlpacaEval\textcolor{mylightbluetext}{$\downarrow$}}   & {MCTest\textcolor{mylightbluetext}{$\downarrow$}}  & {CMRC (CN)} & {AlpacaEval (JP)}   & Avg. Impro. 
\\ 
\midrule
\tabalign{}  & \tabalign{ASVD-I} & 9.51   & 27.11  & 25.81   & 17.89 & 9.45 & 3.06  & 688.23 & 813.23 & --\\
\cmidrule{1-11}   
\tabalign{$k_1=0.99k$}  
& \tabalign{\nestdecom } 
& {9.52} (\textcolor{mylightbluetext}{$\uparrow$0.1\%})     & {25.86} (\textcolor{mylightbluetext}{$\downarrow$4.6\%})
& {24.97} (\textcolor{mylightbluetext}{$\downarrow$3.3\%})  & {16.45} (\textcolor{mylightbluetext}{$\downarrow$8.0\%})    
& {9.13} (\textcolor{mylightbluetext}{$\downarrow$3.4\%})   & {3.03} (\textcolor{mylightbluetext}{$\downarrow$1.0\%}) 
& {605.32} (\textcolor{mylightbluetext}{$\downarrow$12.0\%}) & {574.60 } (\textcolor{mylightbluetext}{$\downarrow$32.7\%}) & 9.3\%\\
\tabalign{$k_1=0.95k$}  & \tabalign{\nestdecom } 
& {9.64} (\textcolor{mylightbluetext}{$\uparrow$1.4\%})       & \textbf{25.19} (\textcolor{mylightbluetext}{$\downarrow$7.1\%})
& {24.41} (\textcolor{mylightbluetext}{$\downarrow$5.4\%})    & {15.74} (\textcolor{mylightbluetext}{$\downarrow$12.1\%})    
& {8.85} (\textcolor{mylightbluetext}{$\downarrow$6.3\%})     & {3.02} (\textcolor{mylightbluetext}{$\downarrow$1.3\%}) 
& {577.63 } (\textcolor{mylightbluetext}{$\downarrow$16.1\%}) & {367.79} (\textcolor{mylightbluetext}{$\downarrow$54.8\%}) & 14.7\%\\
\tabalign{$k_1=0.90k$}  
& \tabalign{\nestdecom } 
& {9.77} (\textcolor{mylightbluetext}{$\uparrow$2.7\%})        & {25.40} (\textcolor{mylightbluetext}{$\downarrow$6.3\%})
& {24.03} (\textcolor{mylightbluetext}{$\downarrow$6.9\%})     & \textbf{14.68} (\textcolor{mylightbluetext}{$\downarrow$17.9\%})    
& {8.68} (\textcolor{mylightbluetext}{$\downarrow$8.1\%})      & {3.00} (\textcolor{mylightbluetext}{$\downarrow$2.0\%}) 
& {412.65} (\textcolor{mylightbluetext}{$\downarrow$40.0\%})  & {335.52} (\textcolor{mylightbluetext}{$\downarrow$58.7\%}) & 20.0\%\\
\tabalign{$k_1=0.85k$}  
& \tabalign{\nestdecom } 
& {9.93} (\textcolor{mylightbluetext}{$\uparrow$4.4\%})               & {25.78} (\textcolor{mylightbluetext}{$\downarrow$4.9\%})
& \textbf{23.98} (\textcolor{mylightbluetext}{$\downarrow$7.1\%})     & {14.74} (\textcolor{mylightbluetext}{$\downarrow$17.6\%})    
& {8.62} (\textcolor{mylightbluetext}{$\downarrow$8.8\%})             & \textbf{2.98} (\textcolor{mylightbluetext}{$\downarrow$2.6\%}) 
& {325.48} (\textcolor{mylightbluetext}{$\downarrow$52.7\%})     & {233.93} (\textcolor{mylightbluetext}{$\downarrow$71.2\%}) & 23.6\%\\
\tabalign{$k_1=0.80k$}  
& \tabalign{\nestdecom } 
& {10.09} (\textcolor{mylightbluetext}{$\uparrow$6.1\%})              & {26.19} (\textcolor{mylightbluetext}{$\downarrow$3.4\%})
& {24.09} (\textcolor{mylightbluetext}{$\downarrow$6.7\%})            & {15.78} (\textcolor{mylightbluetext}{$\downarrow$11.8\%})    
& \textbf{8.59} (\textcolor{mylightbluetext}{$\downarrow$9.1\%})      & \textbf{2.98} (\textcolor{mylightbluetext}{$\downarrow$2.6\%}) 
& \textbf{250.17} (\textcolor{mylightbluetext}{$\downarrow$63.7\%}) & \textbf{204.90} (\textcolor{mylightbluetext}{$\downarrow$74.8\%}) & 24.6\%\\
\midrule
\end{tabular}
}
\captionof{table}{Zero-shot performance of LLaMA-7B compressed using NSVD-I and ASVD-I under a 30\% compression ratio with different values of $k_1$, across \textcolor{black}{eight} language modeling datasets (measured by perplexity  (\textcolor{mylightbluetext}{$\downarrow$})). The best performance is highlighted in bold. The relative performance gain compared to ASVD-I is indicated in blue text with bracket.
In all scenarios, we apply  NSVD-I with $k_1+k_2\equiv k$, where $k$ is the low-rank parameter used in  ASVD approaches.
}
\label{tab:nsce_ppl2}
\end{table}

\subsection{Performance under different values of $k_1$ for NID} 
We previously mentioned that the second step in the nested decomposition approach shown in \eqref{equation:nesterd_lowapp} is flexible.
To explore this flexibility, we compare the results obtained using low-rank interpolative decomposition \citep{martinsson2011randomized, lu2022bayesian}.
Table~\ref{tab:nsce_ppl3} summarizes the findings using NID-I under a 30\% compression ratio.
For PTB, C4, SNIPS, AlpacaEval, and MCTest datasets, containing English sentences (same as the calibration set of WikiText-2), we observe that NID-I can improve the performance even with a small value of $k_2$ ($k_2=0.01k$ and $k_1=0.99k$). This suggests that the nested decomposition approach can be effectively applied with a small value of $k_2$ when a more economical low-rank approximation method is desired (here, the ID approach rather than the SVD method).
However, when evaluating the  CMRC (CN) set, which contains Chinese sentences and is expected to have significantly different activations from the calibration set, the NID-I approach may not perform as well. In such cases, NSVD is necessary to achieve better results.
On average, NID-I with $k_1=0.95k$ reduces  perplexity  by 0.9\% under a 30\% compression ratio.

\begin{table}[H]
\centering
\resizebox{1.\textwidth}{!}{
\begin{tabular}{ccc|ccccccc|cc}
\midrule
\tabalign{\textsc{$k_1$}}      & \tabalign{\textsc{Method}}    & WikiText-2\textcolor{mylightbluetext}{$\downarrow$} & PTB\textcolor{mylightbluetext}{$\downarrow$} & C4\textcolor{mylightbluetext}{$\downarrow$}   &   {SNIPS\textcolor{mylightbluetext}{$\downarrow$}}   & {AlpacaEval\textcolor{mylightbluetext}{$\downarrow$}}   & {MCTest\textcolor{mylightbluetext}{$\downarrow$}}   & {CMRC (CN)} & {AlpacaEval (JP)}   & Avg. Impro.
\\ 
\midrule  
\tabalign{}  & \tabalign{ASVD-I} & 9.51   & 27.11  & 25.81   & 17.89 & 9.45 & 3.06 & \textbf{688.23} &  813.23& -- \\
\midrule 
\tabalign{$k_1=0.99k$}  
& \tabalign{\nidi} 
& {9.54} (\textcolor{mylightbluetext}{$\uparrow$0.3\%})            & \textbf{26.18} (\textcolor{mylightbluetext}{$\downarrow$3.4\%})
& \textbf{25.40} (\textcolor{mylightbluetext}{$\downarrow$1.6\%})  & {17.77} (\textcolor{mylightbluetext}{$\downarrow$0.7\%})    
& \textbf{9.32} (\textcolor{mylightbluetext}{$\downarrow$1.2\%})   & {3.01} (\textcolor{mylightbluetext}{$\downarrow$1.6\%}) 
& {852.63} (\textcolor{mylightbluetext}{$\uparrow$24.9\%}) & {950.00} (\textcolor{mylightbluetext}{$\uparrow$16.8\%}) & $-4.7$\%
\\
\tabalign{$k_1=0.95k$}  & \tabalign{\nidi} 
& {9.84} (\textcolor{mylightbluetext}{$\uparrow$3.5\%})       & {26.59} (\textcolor{mylightbluetext}{$\downarrow$1.9\%})
& {25.98} (\textcolor{mylightbluetext}{$\uparrow$0.6\%})      & \textbf{17.54} (\textcolor{mylightbluetext}{$\downarrow$2.0\%})    
& {9.52} (\textcolor{mylightbluetext}{$\uparrow$0.7\%})       & \textbf{3.00} (\textcolor{mylightbluetext}{$\downarrow$2.0\%}) 
& {798.42} (\textcolor{mylightbluetext}{$\uparrow$16.0\%})    & \textbf{670.99} (\textcolor{mylightbluetext}{$\downarrow$17.5\%}) & 0.9\%
\\
\tabalign{$k_1=0.90k$}  
& \tabalign{\nidi} 
& {10.31} (\textcolor{mylightbluetext}{$\uparrow$8.4\%})      & {28.71} (\textcolor{mylightbluetext}{$\uparrow$5.9\%})
& {27.89} (\textcolor{mylightbluetext}{$\uparrow$8.1\%})      & {19.39} (\textcolor{mylightbluetext}{$\uparrow$8.4\%})    
& {10.20} (\textcolor{mylightbluetext}{$\uparrow$7.9\%})      & {3.10} (\textcolor{mylightbluetext}{$\uparrow$1.3\%})
& {838.61} (\textcolor{mylightbluetext}{$\uparrow$21.9\%})    & {686.28} (\textcolor{mylightbluetext}{$\downarrow$15.6\%}) & $-5.4$\%
\\  \midrule
\end{tabular}
}
\captionof{table}{Zero-shot performance of LLaMA-7B compressed using NID-I and ASVD-I under a 30\% compression ratio with different values of $k_1$, across \textcolor{black}{eight} language modeling datasets (measured by perplexity  (\textcolor{mylightbluetext}{$\downarrow$})). The best performance is highlighted in bold. The relative performance gain compared to ASVD-I is indicated in blue text with bracket.
In all scenarios, we apply NSVD-I with $k_1+k_2\equiv k$, where $k$ is the low-rank parameter used in  ASVD approaches.
}
\label{tab:nsce_ppl3}
\end{table}

\subsection{Performance on different LLMs} 
To evaluate the effectiveness of NSVD across various LLMs, we compare its performance against several baseline methods using three models from different LLM families---OPT-6.7B (OPT family), Vicuna-7B (LLaMA family), and Mistral-7B (Mistral family)---under a 30\% compression ratio on WikiText-2 (the calibration dataset), PTB, C4, SNIPS, Alpaca,  MCTest, CMRC (CN), and  AlpacaEval (JP) datasets. As shown in Table~\ref{tab:nsce_ppl4_diffmodels}, 
ASVD-I generally performs better than ASVD-0, except when using Vicuna-7B on the PTB and AlpacaEval (JP) datasets, where it performs worse.
NSVD-I, on the other hand, consistently outperforms  ASVD-0 and ASVD-I on all three LLMs except one instance (with a minor degradation), and exhibits more stable performance across different LLMs.
Specifically, NSVD-I achieves average performance improvements of 27.6\%, 4.4\%, and 30.1\% for the Vicuna-7B, Mistral-7B, and OPT-6.7B models, respectively.

\begin{table}[H]
\centering
\resizebox{1.\textwidth}{!}{
\begin{tabular}{ccc|ccccccc|cc}
\midrule
\tabalign{\textsc{Model}}      & \tabalign{\textsc{Method}}    & WikiText-2\textcolor{mylightbluetext}{$\downarrow$} & PTB\textcolor{mylightbluetext}{$\downarrow$} & C4\textcolor{mylightbluetext}{$\downarrow$}   &   {SNIPS\textcolor{mylightbluetext}{$\downarrow$}}   & {AlpacaEval\textcolor{mylightbluetext}{$\downarrow$}}   & {MCTest\textcolor{mylightbluetext}{$\downarrow$}}    & {CMRC (CN)} & {AlpacaEval (JP)}   & Avg. Impro.
\\ \midrule 
\tabalign{}   & \tabalign{ASVD-0} & 283.29   & 2861.49  & 345.30   & 142.17 & 122.25 & 20.38  & 1018.54 & 534.31 & --\\
\tabalign{Vicuna-7B}   & \tabalign{ASVD-I} & 14.23   & 3022.79  & 42.16   & 26.05 & 11.21 & 3.63  & 912.86 & 565.20 & --\\
\cmidrule{2-11} 
\tabalign{}  
& \tabalign{\nestdecom } 
& {14.36} (\textcolor{mylightbluetext}{$\uparrow$0.9\%})            & \textbf{1376.72} (\textcolor{mylightbluetext}{$\downarrow$51.9\%})
& \textbf{38.53} (\textcolor{mylightbluetext}{$\downarrow$8.6\%})   & \textbf{24.32} (\textcolor{mylightbluetext}{$\downarrow$6.6\%})    
& \textbf{10.12} (\textcolor{mylightbluetext}{$\downarrow$9.7\%})   & \textbf{3.50} (\textcolor{mylightbluetext}{$\downarrow$3.6\%}) 
& \textbf{405.33} (\textcolor{mylightbluetext}{$\downarrow$55.6\%}) & \textbf{253.98} (\textcolor{mylightbluetext}{$\downarrow$55.1\%}) & 27.6\%
\\
\midrule
\tabalign{}   & \tabalign{ASVD-0} & 320.73   & 1177.23  & 479.52   & 155.27 & 183.24 & 20.15  & 1679.11 & 884.30 & --\\
\tabalign{Mistral-7B}   & \tabalign{ASVD-I} & 48.48   & 135.40  & 58.28   & 18.12 & 21.29 & 4.15  & 18.63 & \textbf{14.93} & -- \\
\cmidrule{2-11} 
\tabalign{}  
& \tabalign{\nestdecom } 
& {42.33} (\textcolor{mylightbluetext}{$\downarrow$12.7\%})            & \textbf{133.68} (\textcolor{mylightbluetext}{$\downarrow$1.3\%})
& \textbf{53.20} (\textcolor{mylightbluetext}{$\downarrow$8.7\%})   & \textbf{17.37} (\textcolor{mylightbluetext}{$\downarrow$4.1\%})    
& \textbf{19.39} (\textcolor{mylightbluetext}{$\downarrow$8.9\%})   & \textbf{3.85} (\textcolor{mylightbluetext}{$\downarrow$7.2\%}) 
& \textbf{18.29} (\textcolor{mylightbluetext}{$\downarrow$1.8\%})   & {15.10} (\textcolor{mylightbluetext}{$\uparrow$1.2\%}) & 4.4\% 
\\
\midrule
\tabalign{}   & \tabalign{ASVD-0} & 98.99   & 109.48  & 88.69   & 42.15 & 28.58 & 5.48  & 27.04 & 130.86 & --\\
\tabalign{OPT-6.7B}   & \tabalign{ASVD-I} & 27.66   & 37.25  & 40.06   & 26.20 & 14.75 & 3.47  & 14.03 & 32.35 & --\\
\cmidrule{2-11} 
\tabalign{}  
& \tabalign{\nestdecom } 
& {19.55} (\textcolor{mylightbluetext}{$\downarrow$29.3\%})            & \textbf{25.05} (\textcolor{mylightbluetext}{$\downarrow$32.8\%})
& \textbf{26.13} (\textcolor{mylightbluetext}{$\downarrow$34.8\%})   & \textbf{16.28} (\textcolor{mylightbluetext}{$\downarrow$37.9\%})    
& \textbf{9.73} (\textcolor{mylightbluetext}{$\downarrow$34.0\%})   & \textbf{2.85} (\textcolor{mylightbluetext}{$\downarrow$17.9\%}) 
& \textbf{11.69} (\textcolor{mylightbluetext}{$\downarrow$16.7\%}) & \textbf{20.50} (\textcolor{mylightbluetext}{$\downarrow$36.6\%}) & 30.1\%
\\
\midrule
\end{tabular}
}
\captionof{table}{Zero-shot performance of Vicuna-7B, Mistral-7B, and OPT-6.7B compressed by NSVD-I and baselines under a 30\% compression ratio, across \textcolor{black}{eight} language modeling datasets (measured by perplexity  (\textcolor{mylightbluetext}{$\downarrow$})). The best performance is highlighted in bold. The relative performance gain compared to the best-performing baseline (here, ASVD-0 or ASVD-I) is indicated in blue text with bracket.
In all scenarios, we apply NSVD-I with $k_1=0.95\cdot k$, $k_2=k-k_1$, where $k$ is the low-rank parameter used in  ASVD approaches.
}
\label{tab:nsce_ppl4_diffmodels}
\end{table}

\subsection{Performance on different scales of LLMs}
Finally, to further investigate  the generality of NSVD across different scales of LLMs, we compare the performance of NSVD with the baselines on \textcolor{black}{three} different models from the LLaMA family---LLaMA-7B, LLaMA-13B, and LLaMA-30B---under a 30\% compression ratio on WikiText-2 (the calibration dataset), PTB, C4, SNIPS, Alpaca,  MCTest, CMRC (CN), and AlpacaEval (JP) datasets. As shown in Table~\ref{tab:nsce_ppl_diffscale},  NSVD  consistently outperforms  ASVD-0 and ASVD-I across all  \textcolor{black}{three} LLMs of varying scales. On average, NSVD-I improves  performance by 14.7\%, 13.4\%, and 3.1\% for LLaMA-7B, LLaMA-13B, and LLaMA-30B models, respectively.

\begin{table}[H]
\centering
\resizebox{1.\textwidth}{!}{
\begin{tabular}{ccc|ccccccc|cc}
\midrule
\tabalign{\textsc{Model}}      & \tabalign{\textsc{Method}}    & WikiText-2\textcolor{mylightbluetext}{$\downarrow$} & PTB\textcolor{mylightbluetext}{$\downarrow$} & C4\textcolor{mylightbluetext}{$\downarrow$}   &   {SNIPS\textcolor{mylightbluetext}{$\downarrow$}}   & {AlpacaEval\textcolor{mylightbluetext}{$\downarrow$}}   & {MCTest\textcolor{mylightbluetext}{$\downarrow$}}    & {CMRC (CN)} & {AlpacaEval (JP)}   & Avg. Impro.
\\ \midrule 
\toppp  & \tabalign{ASVD-0} & 228.77 & 326.20 & 312.18 & 115.65 & 157.33 & 42.37  & 1406.88 & 1359.26& --         \\
\tabalign{LLaMA-7B}  & \tabalign{ASVD-I} & 9.51   & 27.11  & 25.81   & 17.89 & 9.45 & 3.06 & 688.23 & 813.23& --\\
\cmidrule{2-11}   
\toppp  & \tabalign{\nestdecom} & {9.64} (\textcolor{mylightbluetext}{$\uparrow$1.4\%})  & \textbf{25.19} (\textcolor{mylightbluetext}{$\downarrow$7.1\%})& \textbf{24.41} (\textcolor{mylightbluetext}{$\downarrow$5.4\%})    
& \textbf{15.74} (\textcolor{mylightbluetext}{$\downarrow$12.1\%})    
& \textbf{8.85} (\textcolor{mylightbluetext}{$\downarrow$6.3\%})    & \textbf{3.02} (\textcolor{mylightbluetext}{$\downarrow$1.3\%})  
& \textbf{577.63} (\textcolor{mylightbluetext}{$\downarrow$16.1\%}) & \textbf{367.79} (\textcolor{mylightbluetext}{$\downarrow$54.8\%}) & 14.7\%
\\
\midrule
\tabalign{}   & \tabalign{ASVD-0} & 90.21   & 117.07  & 111.96    & 57.47 & 41.06 & 10.02 & 282.47 & 278.47 & --\\
\tabalign{LLaMA-13B}   & \tabalign{ASVD-I} & 7.63   & 17.84  & 18.87   & 14.58 & 7.22 & 2.59  & 154.13 & 82.42 & --\\
\cmidrule{2-11} 
\tabalign{}  
& \tabalign{\nestdecom } 
& {7.69} (\textcolor{mylightbluetext}{$\uparrow$0.8\%})             & \textbf{17.10} (\textcolor{mylightbluetext}{$\downarrow$4.1\%})
& \textbf{17.30} (\textcolor{mylightbluetext}{$\downarrow$8.3\%})   & \textbf{12.11} (\textcolor{mylightbluetext}{$\downarrow$16.9\%})    
& \textbf{6.66} (\textcolor{mylightbluetext}{$\downarrow$7.8\%})   & \textbf{2.53} (\textcolor{mylightbluetext}{$\downarrow$2.3\%}) 
& \textbf{102.39} (\textcolor{mylightbluetext}{$\downarrow$33.6\%})  & \textbf{65.00} (\textcolor{mylightbluetext}{$\downarrow$21.1\%}) & 13.4\%
\\
\midrule
\tabalign{}   & \tabalign{ASVD-0} & 27.76   & 35.59  & 35.90    & 18.28 & 13.11 & 3.40 & 51.71 & 77.66 & --\\
\tabalign{LLaMA-30B}   & \tabalign{ASVD-I} & 6.35   & 11.41 & 13.26    & 8.24 & 5.21 &  2.25  & 15.81 & 17.59 & --\\
\cmidrule{2-11} 
\tabalign{}  
& \tabalign{\nestdecom } 
& {6.38 } (\textcolor{mylightbluetext}{$\uparrow$0.5\%})             & \textbf{11.05 } (\textcolor{mylightbluetext}{$\downarrow$3.2\%})
& \textbf{12.55} (\textcolor{mylightbluetext}{$\downarrow$5.4\%})   & \textbf{7.95} (\textcolor{mylightbluetext}{$\downarrow$3.5\%})    
& \textbf{5.00} (\textcolor{mylightbluetext}{$\downarrow$4.0\%})   & \textbf{2.22} (\textcolor{mylightbluetext}{$\downarrow$1.3\%}) 
& \textbf{15.37} (\textcolor{mylightbluetext}{$\downarrow$2.8\%})  & \textbf{17.29} (\textcolor{mylightbluetext}{$\downarrow$1.7\%}) & 3.1\%
\\
\midrule
\end{tabular}
}
\captionof{table}{Zero-shot performance of LLaMA-7B, LLaMA-13B, and LLaMA-30B compressed using NSVD-I and baselines under a 30\% compression ratio, across \textcolor{black}{eight} language modeling datasets (measured by perplexity  (\textcolor{mylightbluetext}{$\downarrow$})). The best performance is highlighted in bold. The relative performance gain compared to the best-performing baseline (here, ASVD-0 or ASVD-I) is indicated in blue text with bracket.
In all scenarios, we apply NSVD-I with $k_1=0.95\cdot k$, $k_2=k-k_1$, where $k$ is the low-rank parameter used in  ASVD approaches.
}
\label{tab:nsce_ppl_diffscale}
\end{table}

\section{Conclusion}
In this paper, we have presented a method for improving low-rank compression in LLMs  with nested matrix decomposition. Our approach leverages the same amount of compression ratio as  activation-aware SVD approaches, making it both cost-effective and practical for LLM compression. 
We have demonstrated the effectiveness of our method through extensive experiments using different datasets and different large language models from three distinct LLM families and three different scales.
Our findings indicate that ASVD-II performs comparably to ASVD-I in all scenarios with minor numerical benefits from pseudo-inverses. 
However, the NSVD approaches, which use nested matrix decomposition and maintain the same computational complexity, provide superior low-rank approximation. This results in improved overall performance, lower perplexity scores, and more robust and stable performance across different datasets.

\setcitestyle{numbers}
\bibliography{bib}
\bibliographystyle{iclr}

\end{document}